\documentclass[twoside]{article}

\usepackage[accepted]{aistats2019}

\usepackage[round]{natbib}

\usepackage[utf8]{inputenc} 
\usepackage[T1]{fontenc}  
\usepackage{hyperref}       
\usepackage{url}            
\usepackage{booktabs}       
\usepackage{amsfonts}       
\usepackage{amsthm}
\usepackage{nicefrac}       
\usepackage{microtype}      
\usepackage{amsmath}
\usepackage[linesnumbered,ruled]{algorithm2e}
\usepackage{color}
\usepackage{float}
\usepackage{enumitem}
\usepackage{graphicx, xcolor}

\newtheorem{theorem}{Theorem}[section]

\newtheorem{proposition}[theorem]{Proposition}

\newtheorem{definition}{Definition}[section]

\newcommand\independent{\protect\mathpalette{\protect\independenT}{\perp}}
\def\independenT#1#2{\mathrel{\rlap{$#1#2$}\mkern2mu{#1#2}}}

\newcommand{\Wb}{\boldsymbol{W}}
\newcommand{\Xb}{\boldsymbol{X}}

\begin{document}

\twocolumn[

\aistatstitle{Knockoffs for the Mass: New Feature Importance Statistics with False Discovery Guarantees}

\aistatsauthor{Jaime Roquero Gimenez \And Amirata Ghorbani \And James Zou}

\aistatsaddress{ Department of Statistics\\
  Stanford University\\
  Stanford, CA 94305 \\
  \texttt{roquero@stanford.edu} \\ 
  \And  Department of Electrical Engineering \\
  Stanford University\\
  Stanford, CA 94305 \\
  \texttt{amiratag@stanford.edu} 
  \And Department of Biomedical \\ Data Science \\
  Stanford University\\
  Stanford, CA 94305 \\
  \texttt{jamesz@stanford.edu}\\ } ]

\begin{abstract}
An important problem in machine learning and statistics is to identify features that causally affect the outcome. This is often impossible to do from purely observational data, and a natural relaxation is to identify features that are correlated with the outcome even conditioned on all other observed features. For example, we want to identify that smoking really is correlated with cancer conditioned on demographics. The knockoff procedure is a recent breakthrough in statistics that, in theory, can identify truly correlated features while guaranteeing that false discovery rate is controlled. The idea is to create synthetic data—knockoffs—that capture correlations among the features. However, there are substantial computational and practical challenges to generating and using knockoffs. This paper makes several key advances that enable knockoff application to be more efficient and powerful. We develop an efficient algorithm to generate valid knockoffs from Bayesian Networks. Then we systematically evaluate knockoff test statistics and develop new statistics with improved power. The paper combines new mathematical guarantees with systematic experiments on real and synthetic data. 
\end{abstract}

\section{INTRODUCTION}

Identifying important features is an ubiquitous problem in machine learning and statistics. In relatively simple settings, the importance of a given feature is measured via the fitted parameters of some model. Consider Generalized Linear Models (GLM) for example, users often add a LASSO penalty \citep{LASSO} to promote sparsity in the coefficients, and subsequently select those features with non-zero coefficients. Step-wise procedures where we sequentially modify the model are another way of doing feature selection \citep{mallowscp,aicModel,bic}. 

These standard methods are all plagued by correlations between the features: a feature that is not really relevant for the outcome, in a precise sense we will define in Section~2, can be selected by LASSO or Step-wise procedure, because it is correlated with relevant features. The feature selection problem becomes even more difficult in more complex settings where we may not have a clean parametric model relating input features to the labels. Moreover, in these settings we usually lack statistical guarantees on the validity of the selected features. Finally, even procedures with statistical guarantees usually depend on having \emph{valid $p$-values}, which are based on a correct modeling of $Y|X$ and (sometimes) assume some asymptotic regime. However, there are many common settings where these assumptions fail and we cannot perform inference based on those $p$-values \citep{sur2018modern}.

A powerful new approach called \emph{Model-X knockoff procedure} \citep{KN2} has recently emerged to deal with these issues. This method introduces a new paradigm: we no longer assume any model for the distribution of $Y|X$ in order to do feature selection (and therefore do not compute $p$-values), but we assume that we have full knowledge of the feature distribution $P^X$ -- or at least we can accurately model it. This knowledge of the ground truth $P^X$ allows us to sample new \emph{knockoff} variables $\tilde{X}$ satisfying some precise distributional conditions. Although we make no assumption on $Y|X$, we can run any procedure on the data set extended with the knockoffs and perform feature selection while controlling the False Discovery Rate (FDR). Statistical guarantees no longer rely on \emph{valid $p$-values} whenever our model for $Y|X$ is correct, but exclusively on having \emph{valid knockoffs}, obtained whenever we know $P^X$ (or our model for $P^X$ is correct).

There are two main obstacles to using knockoffs in practice (i.e. ``for the mass''): 1) tractably generating valid knockoffs and 2) once we have generated knockoffs, computing powerful test statistics. Current tractable methods for generating knockoffs are restricted to the settings where $X$ is modeled as a multivariate Gaussian \citep{KN2} or as the set of observed nodes in a Hidden Markov Model (HMM) \citep{KN3}. Even though knockoff methods enjoy some robustness properties when approximating $P^X$ \citep{KN4}, the guarantee on FDR control breaks down even in very simple settings if these approximations turn out to be too crude, as we demonstrate in this paper. Sequential conditional independent pairs algorithm (SCIP) \citep{KN2} is a universal knockoff sampling scheme but is computationally intractable as soon as we model $P^X$ by more complex distributions. Constructing tractable knockoff sampling procedures for more flexible classes of distributions used to model $P^X$ is essential to improve the validity of knockoffs, which is always subject to the quality of the approximation. Regarding test statistics, current methods \citep{KN1,KN2,KN3,KN4} focus on LASSO-based statistics to obtain proxies for the importance of any given feature: such methods need to fit a GLM to the data, which is too restrictive in many supervised learning tasks. 

\paragraph{Our Contributions.} First, we formulate a novel tractable procedure for sampling knockoffs in settings where features can be modeled as the observed variables in a Bayesian Network. This allows for great flexibility when modeling the feature distribution $P^X$. We show that this procedure is different from SCIP, which is the current state-of-the-art method to sample knockoffs. We construct valid knockoffs in settings where previous knockoff sampling procedures assumed a very restrictive model for $P^X$ and therefore failed to control FDR, and provide a unified framework for several different sampling procedures. In addition, we systematically evaluate and compare several nonlinear knockoff test statistics which can be applied in general supervised learning problems. We develop a new statistic, \emph{Swap Integral}, which has significantly better power than other methods. These two advances enable us to perform feature selection using black-box classifiers to represent $Y|X$ while still retaining statistical guarantees under correct modeling of $P^X$.

\section{KNOCKOFF BACKGROUND}

We begin by introducing the usual setting of feature selection procedures. We consider the data as a sequence of i.i.d. samples from some unknown joint distribution: $(X_{i1},\dots,X_{id},Y_i) \sim P^{XY}$, $i = 1,\dots,n$. We then define the set of null features $\mathcal{H}_0\subset \{1,\dots,d\}$ by  $j\in \mathcal{H}_0$ if and only if $X_j \independent Y | \Xb_{-j}$ (where the $-j$ subscript indicates all variables except the $j$th). The non-null features, also called alternatives, are important because they capture the truly influential effects: each non-null feature is correlated with the label $Y$ even conditioned on rest of the features. Our goal is to identify these non-null features. Running the knockoff procedure gives us a selected set $\hat{\mathcal{S}}\subset \{1,\dots,d\}$, while controlling for False Discovery Rate (FDR), which stands for the expected rate of false discoveries: $FDR = \mathbb{E}\Big[ \frac{|\hat{\mathcal{S}}\cap \mathcal{H}_0|}{|\hat{\mathcal{S}}|\vee 1} \Big]$.

Assuming we know the ground truth for the distribution $P^X$ (we'll come back to the realistic setting where we only have samples of $X$), the first step of the knockoff procedure is to obtain a \emph{knockoff} sample $\tilde{X}$ that satisfies the following conditions: 

\begin{definition}[Knockoff sample]
A knockoff sample $\tilde{X}$ of a d-dimensional random variable $X$ is a d-dimensional random variable such that two properties are satisfied:
\begin{itemize}[noitemsep, topsep=0pt]
\item Conditional independence: \hspace{3mm}$\tilde{X} \independent Y | X$ 
\item Exchangeability :  
\[[X,\tilde{X}]_{swap(S)}\; \stackrel{d}{=}\; [X,\tilde{X}]\qquad  \forall S \subset \{1,\dots,d\}
\]
\end{itemize}
where $\stackrel{d}{=}$ stands for equality in distribution and the notation $[X,\tilde{X}]_{swap(S)}$ refers to the vector where the original $j$th feature and the $j$th knockoff feature have been transposed whenever $j\in S$.
\end{definition}

The first condition is immediately satisfied as long as knockoffs are sampled conditioned on the sample $X$ without considering any information about $Y$, which will always be the case in our sampling methods. More generally we say that any $f:\mathbb{R}^d \times \mathbb{R}^d \rightarrow \mathbb{R}$  such that $f([x,\tilde{x}]_{swap(S)})=f([x,\tilde{x}])$ satisfies the exchangeability property.

Once we have the knockoff sample $\tilde{X}$, the next step of the procedure constructs \emph{feature statistics} \mbox{$\Wb=(W_1,\dots,W_d)$}, such that a high value for $W_j \in \mathbb{R}$ is evidence that the $j$th feature is non-null. Feature statistics described by \citet{KN2} depend only on $[X,\tilde{X}] \in \mathbb{R}^{N\times 2d},Y\in \mathbb{R}^N$ such that for each $j \in \{1,\dots,d\}$ we can write $W_j = w_j([X,\tilde{X}],Y) $ for some function $w_j$. The only restriction these statistics must satisfy is the \emph{flip-sign property}: swapping the $j$th feature and its corresponding knockoff feature should flip the sign of the statistic $W_j$ while leaving other feature statistics unchanged. More formally, for a subset $S\subset \{1,\dots,d\}$ of features, denoting $[X,\tilde{X}]_{swap(S)}$ the data matrix where the original $j$th variable and its corresponding knockoff have been transposed whenever $j\in S$, we require that $w_j$ satisfies:
\vspace{-8mm}

\begin{equation*}
  w_j([X,\tilde{X}]_{swap(S)},Y)  =\begin{cases}
   - w_j([X,\tilde{X}],Y) , & \hspace{-1mm} \text{if $j\in S$}.\\
   w_j([X,\tilde{X}],Y) , & \hspace{-5mm} \text{otherwise}.
  \end{cases}
\end{equation*}

\vspace{-2mm}

As suggested by \citet{KN2}, the choice of feature statistics can be done in two steps: first, find a statistic $Z=Z([X,\tilde{X}],Y)=(Z_1,\dots,Z_d,\tilde{Z}_1, \dots,\tilde{Z}_d)\in \mathbb{R}^{2d}$ where each coordinate corresponds to the ``importance'' --- hence we will call them \emph{importance scores} ---  of the corresponding feature (either original or knockoff). For example, $Z_j$ would be the absolute value of the regression coefficient of the $j$th feature. 

After obtaining the importance score for the original and knockoff feature, we compute the feature statistic $W_j = Z_j-\tilde{Z}_j$. The intuition is that importance scores of knockoffs serve as a control, such that larger importance score of the $j$th feature compared to that of its knockoff implies larger $W_j$ (and therefore is evidence against the null). Given that feature statistics of the null features behave ``symmetrically'' around 0 by the flip-sign property, the final selection procedure compares, at a given threshold $q>0$ the number of features such that $W_j>q$ vs. the number of those such that $W_j<-q$ \citep{KN2}. By keeping this ratio ``under control'' for a given target FDR that we fix in advance, we obtain an adaptive procedure that selects features whose value $W_j$ is above some adaptive threshold. FDR control is guaranteed with this procedure, as long as the knockoffs are sampled correctly (i.e., satisfying the two conditions), regardless of the choices made to construct the feature statistics $W$ \citep{KN2}.

The power of the knockoff procedure is that it makes \emph{no assumptions} on the relationship between $X$ and $Y$; however, it requires full knowledge of the feature distribution $P^X$ which is often not known and needs to modeled. It is also important to notice that for any model we use for $P^X$, we are required to know how to sample knockoffs from that model distribution. Current classes of distributions having a tractable knockoff sampling method have limited expressivity, and we can find simple settings where using models limited to these classes fails to provide valid knockoffs, and consequently the FDR is not controlled. We provide an example of such setting in Section~\ref{subsection:needmixture}. However, the knockoff procedure is robust to small errors in the estimation of $P^X$ \citep{KN4}. Our method for sampling knockoffs from a Bayesian network allows us to fit a better distribution to the data resulting in a better control of FDR up to some error term with respect to the FDR we target (under the assumption that our model distribution of $P^X$ is perfect, this procedure controls FDR at the target FDR). We discuss in Appendix~\ref{proof:robustness} the reasons why we expect our method to be robust to model misspecification.

Note that setting $\tilde{X} = X$ satisfies the knockoff properties, but no discovery would be made since all of the $W_j$ would be zero. Therefore, in addition to generating valid knockoffs in order to control FDR, it is also important to consider the statistical power of the procedure, which is the probability that a true alternative is selected. This depends not only on the knockoff sampling procedure, but also on the choice of the importance scores. We discuss power in detail in Sections 4 and 5. 

\section{GENERATING KNOCKOFFS FROM LATENT VARIABLE MODELS}
\label{section:GeneratingKnockoffs}

We first show how to generate knockoff samples from a Gaussian mixture model, which is already a new contribution, as a warm-up to our general procedure to generating knockoffs from latent variable models. This section relates just to the first step of the \emph{Model-X} knockoff procedure. We will consider different classes of distributions used to model $P^X$; we do not consider the response $Y$ for this step.

\paragraph{Gaussian Mixture Knockoffs} The setup:
\vspace{-2mm}
\begin{itemize}[leftmargin=*]
\setlength\itemsep{-0.2em}
\item $K$ is a categorical latent variable taking values in $\{1,\dots,l\}$, where $l\geq 2$ is the number of mixtures.
\item $\mu_k \in \mathbb{R}^d$ and $\Sigma_k \in \mathbb{R}^{d \times d}$ are the mean and variances for component $k$, that are known or estimated. 
\item We observe a $\mathbb{R}^d$-valued random variable $X$ with distribution 
\vspace{-8mm}

\[X|K\! =\! k \, \sim P^{X|K}(x|K=k)=\mathcal{N}(x;\mu_k , \Sigma_k)
\]
\end{itemize}
\vspace{-3mm}

For a multivariate Gaussian distribution, we can explicitly write a joint density $P^{X,\tilde{X}|K}$ on $\mathbb{R}^d \times \mathbb{R}^d$ that satisfies the exchangeability property for a given $K$: 
\vspace{-2mm}
\begin{align*}
P^{X,\tilde{X}|K}&(x,\tilde{x}|K=k) =
\\ &
\mathcal{N}\Bigg(\binom{x}{\tilde{x}};\binom{\mu_k}{\mu_k}, \binom{\quad \Sigma_k \qquad \Sigma_k -D_k}{\Sigma_k -D_k \qquad \Sigma_k \quad} \Bigg)
\end{align*}
$D_k$ is a diagonal matrix with the constraint that the joint covariance matrix is positive definite \citep{KN1}. The Gaussian mixture knockoff sampling procedure consists of first sampling the mixture assignment variable from the posterior distribution $P^{K|X}$. The knockoff is then sampled from the conditional distribution of the knockoff given the true variable and the sampled mixture assignment defined as $Q^{\tilde{X}|X,K}(\tilde{x}|x, K=k)$. More explicitly:
\vspace{-6mm}
\begin{align*}
\\&K|X \sim P^{K|X}(k|X) = \frac{\lambda_k \mathcal{N}(X;\mu_k, \Sigma_k)}{\sum_{j=1}^l \lambda_j \mathcal{N}(X;\mu_j, \Sigma_j)}
\\ &\tilde{X} |X,K\! =\! k \, \sim Q^{\tilde{X}|X,K}(\tilde{x}|X, K=k) =  \mathcal{N}\big( \tilde{x} ; \tilde{\mu}_k , \tilde{\Sigma}_k \big)
\\ &\text{where} \;
\begin{cases}
\tilde{\mu}_k = D_k\Sigma_k^{-1} \mu_k + (I_d - D_k\Sigma_k^{-1})X
\\ \tilde{\Sigma}_k = 2D_k - D_k\Sigma_k^{-1}D_k
\end{cases}
\\ &  \text{as $P^{X,\tilde{X}|K}=Q^{\tilde{X}|X,K}P^{X|K}$.}
\end{align*}

\begin{proposition}
\label{proposition:GMM}
The Gaussian mixture knockoff sampling procedure stated above is such that $\tilde{X}$ is a valid knockoff of $X$. 
\end{proposition}

The proof is deferred to Appendix~\ref{proof:gaussianmixture}. Once we compute the necessary parameters $\tilde{\mu}_k, \tilde{\Sigma}_k$ of these conditional distributions, we can sample valid knockoffs for a Gaussian mixture model. Whenever we no longer know the ground truth for $P^X$, Gaussian mixtures can approximate arbitrarily well any continuous distribution \citep{gaussianapprox}, so we now have a very flexible model for the feature distribution $P^X$ from which we know how to tractably sample knockoffs. In a real setting, the validity of our knockoffs will be limited by the quality of the approximation of our model for $P^X$ -- the same way in a parametric model of $Y|X$ the validity of the $p$-values depends on the validity of model and distributional assumptions (of the eventual random noise, on the asymptotic regime, etc). Next we show how to efficiently generate valid knockoffs for general Bayesian Networks, which can be used to accurately model $P^X$ in many settings.

\paragraph{Exchangeable Conjugate} Our knockoff sampling procedures will be based on the following idea. Instead of keeping long range dependencies across all features as SCIP does, we want to exploit the structure given by the local conditional distributions of a Bayesian network: sampling knockoffs becomes as hard as doing probabilistic inference on the Bayesian network, and sampling ``local'' knockoffs based on the local structure. 

\begin{definition}
\label{definition:Conjugate}
Let $U \in \mathcal{U}, \; V \in \mathcal{V}$ be arbitrary multi-dimensional random variables. For any given conditional distribution $P^{U|V}(\cdot|v)$, we say that $Q^{\tilde{U}|U,V}(\cdot|u,v)$ is an \emph{exchangeable conjugate conditional} if it is a probability kernel on $(\mathcal{U}\times \mathcal{V}) \times \mathcal{U}$ such that, for any fixed $v$, $\phi_{v}(u,\tilde{u}) := P^{U|V}(u|v)Q^{\tilde{U}|U,V}(\tilde{u}|u,v)$ satisfies the exchangeability property in $(u,\tilde{u})$.
\end{definition}
\vspace{-2mm}
In other words, if $U$ is sampled conditionally on $V$, and then $\tilde{U}$ is sampled from the probability kernel $Q^{\tilde{U}|U,V}$, then $\tilde{U}$ is a knockoff of $U$ (conditionally on $V$). In some cases it is easy to find examples of such a conjugate conditional distribution. If $U$ is univariate, then $Q^{\tilde{U}|U,V}(\cdot|u,v)=P^{U|V}(\cdot|v)$ is a valid conjugate conditional. More generally, this choice is valid if the coordinates of $U$ are independent conditionally on $V$. We can also construct valid conjugates when the distribution of $U$ conditionally on $V$ is Gaussian \citep{KN2} or a Markov chain \citep{KN3}. 

\paragraph{Sampling Knockoffs for Bayesian Networks} We define a procedure for sampling knockoffs when we assume a generative model for $P^X$. $X$ corresponds to the observed variables in a latent variable model which we assume can be represented as a Bayesian Network (BN): Gaussian Mixture Models, Hidden Markov Models (HMM) \citep{HMM}, Latent Dirichlet Allocation (LDA) \citep{LDA}, Naive Bayes \citep{naivebayes} and many other Bayesian models fit this description. Consider a BN defined by the directed acyclic graph (DAG) $(G,E)$ with cardinality $m=|G|$ where the index is a topological ordering. $X_O$ and $X_H$ refer respectively to the observed variables and the hidden ones: we have restated our initial set of features $X$ as $X_O$, and we want to construct knockoffs by leveraging our knowledge about the DAG. The only restriction is that the observed variables must correspond to nodes in the graph without descendents. It is important to notice that here, the variables associated to a node of the DAG \emph{can be multi-dimensional}, which sometimes simplifies the sampling procedure (an example are knockoffs for HMM, detailed in Appendix~\ref{proof:differenceSCIP}).

We assume that we can sample from $P^{H|O}$ (for simplicity we refer to the random variables by their indices), the conditional probability of the hidden variables given the observed ones. Assume also that, for every node $i$, we can compute the local conditional probability $P^{i|MB(i)}$ of $X_i$ given its Markov blanket $X_{MB(i)}$, and that we can compute a knockoff conjugate distribution $Q^{\tilde{i}|i,MB(i)}$ of $P^{i|MB(i)}$ and sample $\tilde{X}_i$ from it, which will be a ``local'' knockoff. We describe the knockoff generating procedure in Algorithm~\ref{algorithm:BNKnockoffAlgo} and further discuss the assumptions on the inputs on Appendix~\ref{proof:differenceSCIP}.

\begin{algorithm}
    \SetKwInOut{Input}{Input}
    \SetKwInOut{Output}{Output}
    \Input{Initial sample $X_O$, conditional distribution $P^{H|O}$, conjugate conditionals $Q^{\tilde{i}|i,MB(i)}$ for every node $i$.}
    \Output{Knockoff sample $\tilde{X}_O$}
    Sample $X_H \sim P^{H|O}(\cdot | X_O)$\;
   \For{$i = 1$ \KwTo $m$ }{
    Sample $\tilde{X}_{i} \sim Q^{\tilde{i}|i,MB(i)}(\cdot | \tilde{X}_{\{1:i-1\}\cap MB(i)},X_{\{i+1:m\}\cap MB(i)},X_i)$
    }
    \Return $\tilde{X}_O = (\tilde{X}_i)_{i\in O}$
    \caption{Knockoff Sampling Procedure in BN}\label{algorithm:BNKnockoffAlgo}
\end{algorithm}

\begin{theorem}
\label{theorem:BNKnockoffThm}
The distribution of the concatenated random variables $(X_O,\tilde{X}_O)$ satisfies the exchangeability property. That is, $\tilde{X}_O$ is a valid knockoff of $X_O$.
\end{theorem}

The proof is in Appendix~\ref{proof:DAGknockoff}.
This sampling method is not equivalent to the Sequential Conditional Independent Pairs (SCIP) introduced in \citep{KN2}. Consider the simple case with one observed multi-dimensional variable $X$ and one hidden variable $H$. SCIP would sample each coordinate $\tilde{X}_{i}$ of $\tilde{X}$ sequentially by computing (for each knockoff sample) conditional distributions of the joint probability distribution of $(X,\tilde{X}_{j\leq i})$, without leveraging our knowledge of the generative process. That is not computationally feasible in an arbitrary generative model. Another difference with respect to SCIP is that our sampling procedure is not creating a knockoff from the full set of hidden and observed variables: not all ``local'' knockoffs are retained in the final output. That is, $(\tilde{X}_O,\tilde{X}_H)$ is not a valid knockoff for $(X_O,X_H)$. We refer to Appendix~\ref{proof:differenceSCIP} for more details on these differences between methods.

The HMM knockoff sampling procedure \citep{KN3} and the Gaussian mixture sampling above are special cases of Algorithm~\ref{algorithm:BNKnockoffAlgo} (detailed in Appendix~\ref{proof:differenceSCIP}). Naive Bayes would have a straightforward knockoff sampling scheme (as the observed nodes are independent conditionally on the hidden node): given the observed nodes, generate a sample for the hidden node (based on the posterior distribution), then sample independent knockoffs based on the sampled hidden node. We provide in Appendix~\ref{proof:differenceSCIP} an additional example of a DAG knockoff generating procedure based on the LDA model.

\section{POWER ANALYSIS: EVALUATING  IMPORTANCE SCORE METHODS}

\begin{figure}[ht]
\centering
\includegraphics[width=\linewidth]{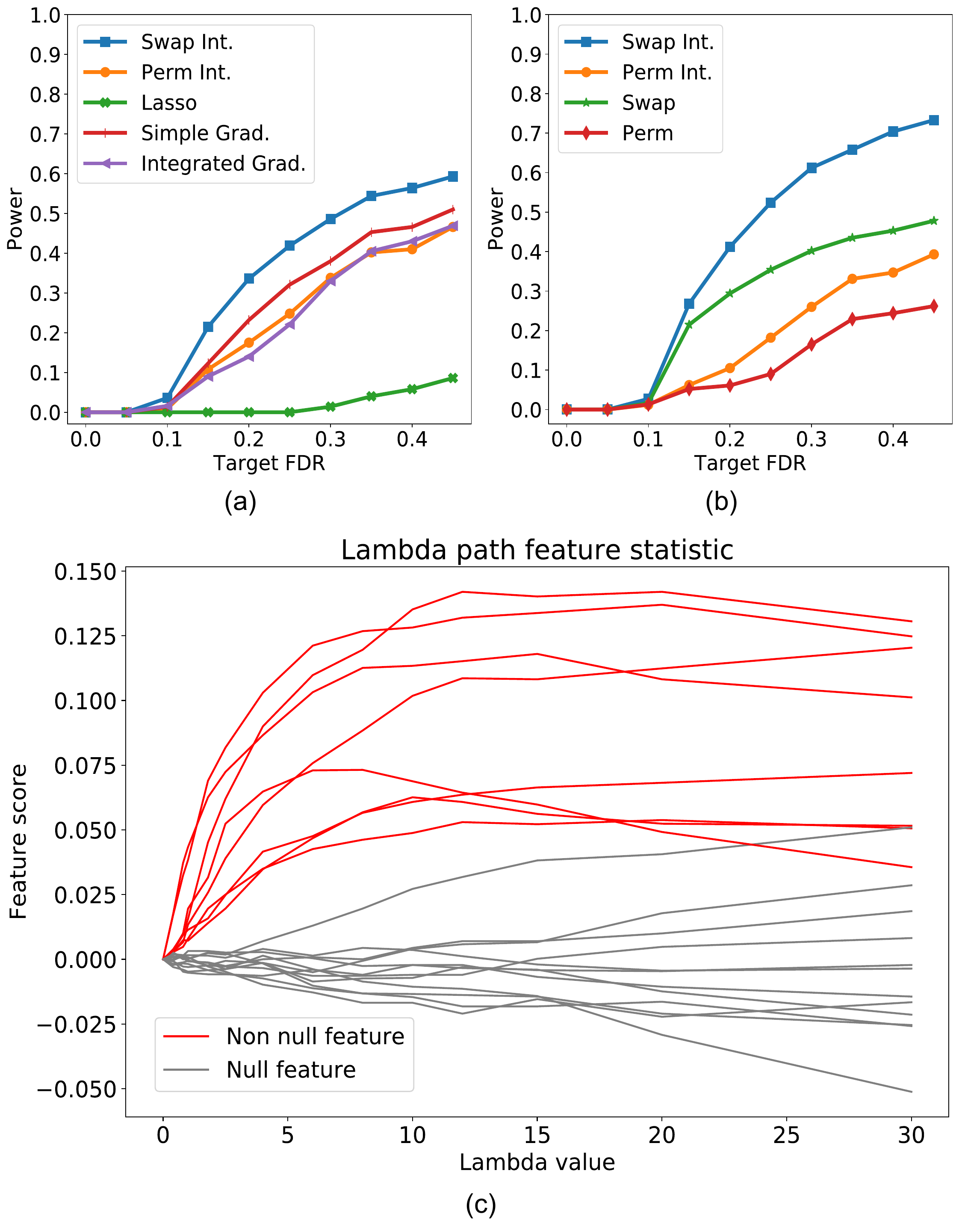}
\vspace{-5mm}
\caption{\textbf{Comparison Between Importance Score Methods} (a) Synthetic data where $Y|X$ is a polynomial. Using a linear LASSO regression model to generate importance scores results in near-zero power; our proposed Swap Integral statistic achieves higher power across all the target FDRs. (b) Swap Integral and Permutation Integral achieve higher power compared respectively to basic Swap and Permutation. (c) Example of lambda path for Swap method.}
\label{fig:lambda}
\end{figure}

The previous section shows how to generate valid knockoffs $\tilde{X}$ for complex latent variable distributions. The next challenge is to compute importance scores for each original and knockoff feature; this was denoted as $Z_j$ and $\tilde{Z}_j$ in Section. 2 (recall that the feature statistics is $W_j = Z_j - \tilde{Z}_j$). While the knockoff procedure guarantees that the FDR is controlled for any importance score, the power of the procedure---i.e. how many true non-null features are discovered---is highly dependent on the method for computing the  importance score. In this section we discuss several importance score methods applicable to a large family of predictive models and also define new ones.

\paragraph{LASSO Coefficient-Difference (LCD)} were used as feature statistics by \citet{KN2}. However, for non-linear $Y|X$, using importance scores based on a linear model assumption will result in low power as LASSO is not able to fit the real input-output relationship (Figure~\ref{fig:lambda}(a)). Complex unknown non-linear input-output relationship is an important restriction if we want to use a parametric model to obtain importance scores through interpretable parameters. 

\paragraph{Saliency Methods} If $Y|X$ is fitted using a neural network, then one could compute saliency scores, a popular metric for importance scores~\citep{lipton2016mythos,ghorbani2017interpretation}. Saliency is computed for each feature of each example, and averaged across all the examples to derive a global importance score. In our experiments, we use Gradient~\citep{baehrens2010explain, simonyan2013deep} and Integrated Gradients~\citep{integrated_gradients} saliency methods.

\paragraph{Importance Scores Based on Accuracy Drops} After a model is trained on the dataset $([X,\tilde{X}],Y)$, for each column of $[X,\tilde{X}]\in \mathbb{R}^{N\times 2d}$, we create a new ``fake'' feature column resulting in the ``fake'' set of features $[X^f,\tilde{X}^f]$, for example by permuting across samples column-wise. Replacing each feature with its fake version one at a time, the accuracy drop can be used as the importance score $Z_j$ for that feature. 
 
Note that achieving a high performance of the classifier on the initial dataset is not necessary to generate these importance scores, as our analysis is based on the accuracy drops. Classifiers with low predictive power can still be useful at determining non-null features, and although they may not make many discoveries, FDR is still controlled.

\paragraph{Permutation and Swapping} Originally introduced for Random Forests \citep{breiman2001random}, the \textbf{Permutation} method creates each column of $[X^f,\tilde{X}^f]$ by simply shuffling the entries of the corresponding column in $[X,\tilde{X}]$.
Although the permutation  method may be effective at creating a new ``fake'' feature independently from $Y$, it comes with an important issue: when evaluating the classifier on the dataset where just one of the features has been shuffled, for general distributions of the features, we may end up with fake data points that are completely off the original data distribution and as a result outside of the learned manifold of the predictive model. The classifier's predictions on these \emph{off-distribution} regions of the input space where it hasn't been trained may be arbitrary. As a result, the performance can potentially drop across all features and their knockoffs, regardless of whether they actually had an impact on $Y$. A simple illustration of this phenomenon has been provided in Appendix~\ref{appendix:schematic}.
 
The main problem occurs with the non-null features, as an unexpectedly large drop in accuracy for the control knockoff feature would mask the non-null importance score, hence limiting the possibility that it is selected. When the original distribution $P^X$ of the features is a mixture of Gaussians, permutation across any of the features will create data points that do not lie on any of the mixture components of $P^X$, and as the number of mixtures increases, the chance for a fake data point to be off-distribution increases. We confirm this phenomenon with synthetic data experiments Figure~\ref{fig:methods}(a-c). Power for the Permutation method strictly decreases as the number of mixture components grows.

To tackle this problem, we introduce the \textbf{Swap} method for importance scores. Instead of shuffling, for each original feature column, we use its knockoff corresponding column as its fake replacement and the other way around: $X^f = \tilde{X}$ and $\tilde{X}^f = X$. Keeping in mind the exchangeability condition when generating knockoffs, this method guarantees that replacing each feature with its fake will result in data points that still belong to the original data distribution, and therefore the problem with fake data points moving to regions of the space unknown to the predictive model is mitigated. 

\paragraph{Path Integration Generates Importance Scores with Better Power} For a given parameter $\lambda \! \geq \! 0$, we compute for each original feature the fake replacement $X^f_j(\lambda) = X_j + \lambda(\tilde{X}_j - X_j)$ (and reciprocally for each knockoff feature). The basic Swap (or Permutation) method corresponds to $\lambda = 1$, and we plot the feature statistics $W$ as a function of the parameter $\lambda$, which we call \emph{lambda path}. 
We give an example of one such path in Figure~\ref{fig:lambda}(c) for a simulation where the classifier is a neural network with 3 fully connected layers, and $Y$ is categorical with 10 classes such that $Y|X$ corresponds to the level sets of a random polynomial of degree 3 of the non-null features, and $X$ comes from a mixture of 10 Gaussians.  Similarly, we can also take lambda path for the  Permutation method, where we compute $X^f_j(\lambda)$ by taking  $\tilde{X}_j$ as the shuffled feature $X_j$. However, our lambda path quickly becomes unstable as for $\lambda=1$ we already get out of the data manifold.

Basic Swap and Permutation (which correspond to $\lambda = 1$) do not always give the best separation between the null features (grey) and the alternate features (red). This motivates taking the integral of the lambda path and use the area under the curve as the feature statistic $W$---we call this \textbf{Swap Integral} and, similarly, \textbf{Permutation Integral}. For all of our experiments, we integrate from $\lambda = 0$ to $\lambda = 10$. Figure~\ref{fig:lambda}(b) shows that, for the same data as used in Figure~\ref{fig:lambda}(c), Swap Integral and Permutation Integral methods are respectively much more powerful than the basic Swap and Permutation. 

Our goal is to identify the classifier's decision boundaries within the data manifold (thus likely to identify relevant features). When considering a lambda path, all comes down to choosing an appropriate direction when translating each feature, and keeping track of the decision boundaries crossed. We may temporarily step out of the data manifold (indeed there is no guarantee that the manifold is convex, our Gaussian mixture model is especially well-suited for this) --and eventually will be out of it. The integral (up to $\lambda = 10$) assigns a higher importance to boundaries crossed early on. As the Swap Integral method ``stays longer'' in the data manifold early on ($\lambda = 1$ is still in it), it performs better than when choosing an arbitrary direction (such as one given by a shuffled vector, i.e. the Permutation Integral). Knockoffs not only are useful for FDR control, they also contain information about the data manifold. The outcome of the following experiments whenever we use Swap Integral does not substantially change whenever $\lambda$ stays in the range 5 to 15.

\section{EXPERIMENTS}

We start by precising that we will not be able to compare our DAG knockoff procedure (for the Gaussian Mixture Model) with the more general SCIP procedure, because the latter has no tractable formulation even for very simple models like a simple multivariate Gaussian.

\paragraph{The Need for Mixture Models} \label{subsection:needmixture} We give a representative example of a simple setting where simply using the empirical covariance matrix to generate the knockoff---i.e. assuming a single Gaussian model---yields invalid knockoffs for which FDR is not controlled. 

We generate data $X$ from a mixture of three correlated Gaussian random vectors in dimension 30. Approximating $P^X$ by a single Gaussian vector yields an empirical covariance matrix which is dominated by the diagonal (and therefore considers $P^X$ as having independent coordinates). We generate a binary outcome $Y$ in a logistic regression setting (such that the log odds of $Y=1$ is a linear combination of the features of X). The linear combination is such that just the first 10 features of $X$ are non-null (i.e. have non zero coefficient, red), the next 10 (null, black) features are correlated with the non-nulls, and the last 10 are independent nulls. During the knockoff procedure, the importance scores $Z$ are the absolute value of the logistic regression coefficients of each feature and the feature statistics $W$ are the difference of importance scores. 

\begin{figure}[ht]
\centering
\includegraphics[width=0.9\linewidth]{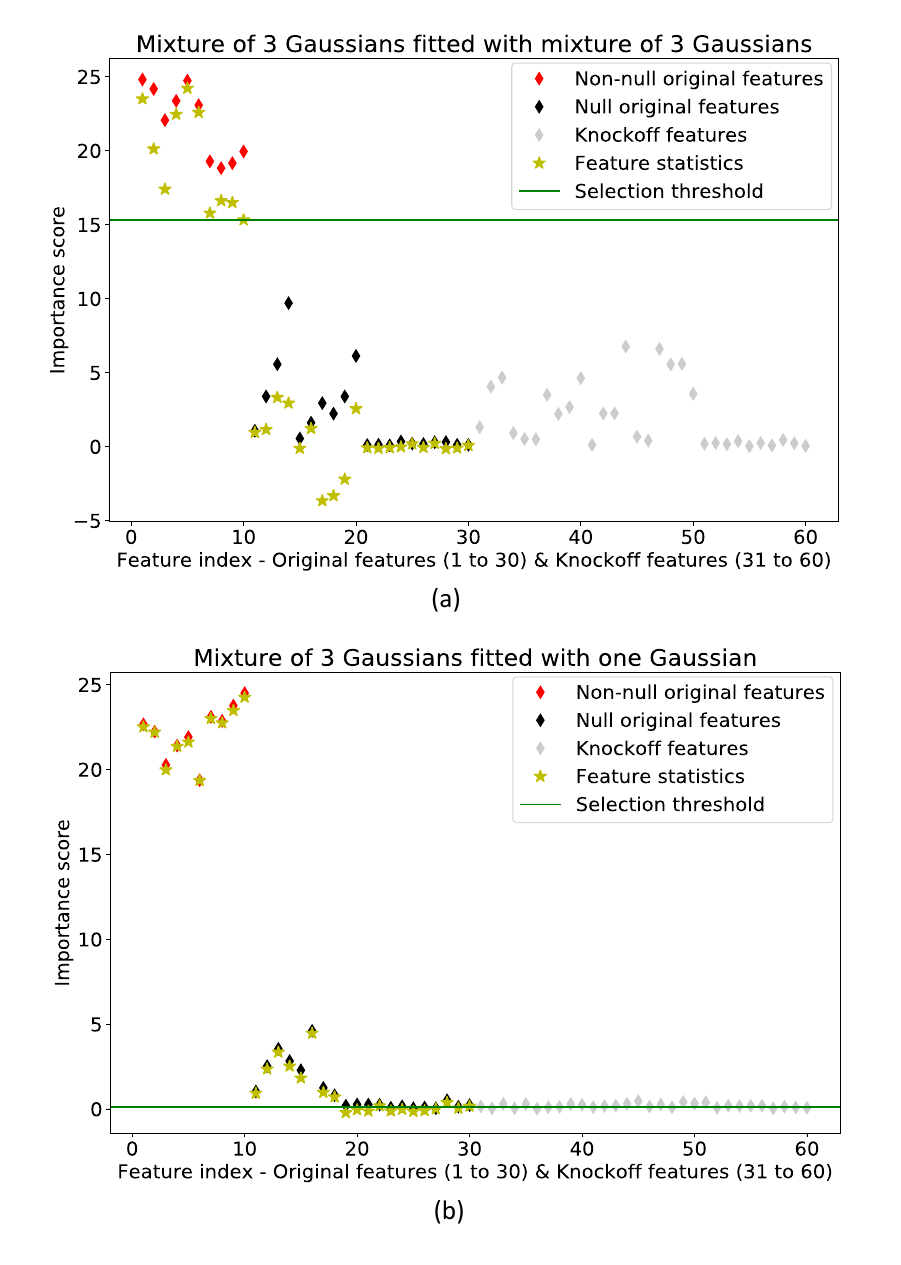}
\vspace{-5mm}
\caption{\textbf{FDR Control Fails for Non Valid Knockoffs:} We plot side by side the importance scores for each feature (original and knockoff, 60 in total) for one run of the simulation, and the feature statistics computed as the difference between importance scores of original vs. knockoff features (plotted in the first 30 indices). Selected features are those whose feature statistic $W$ is above the threshold.}
\label{fig:break}
\end{figure}

\begin{figure}[ht]
\centering
\includegraphics[width=0.85\linewidth]{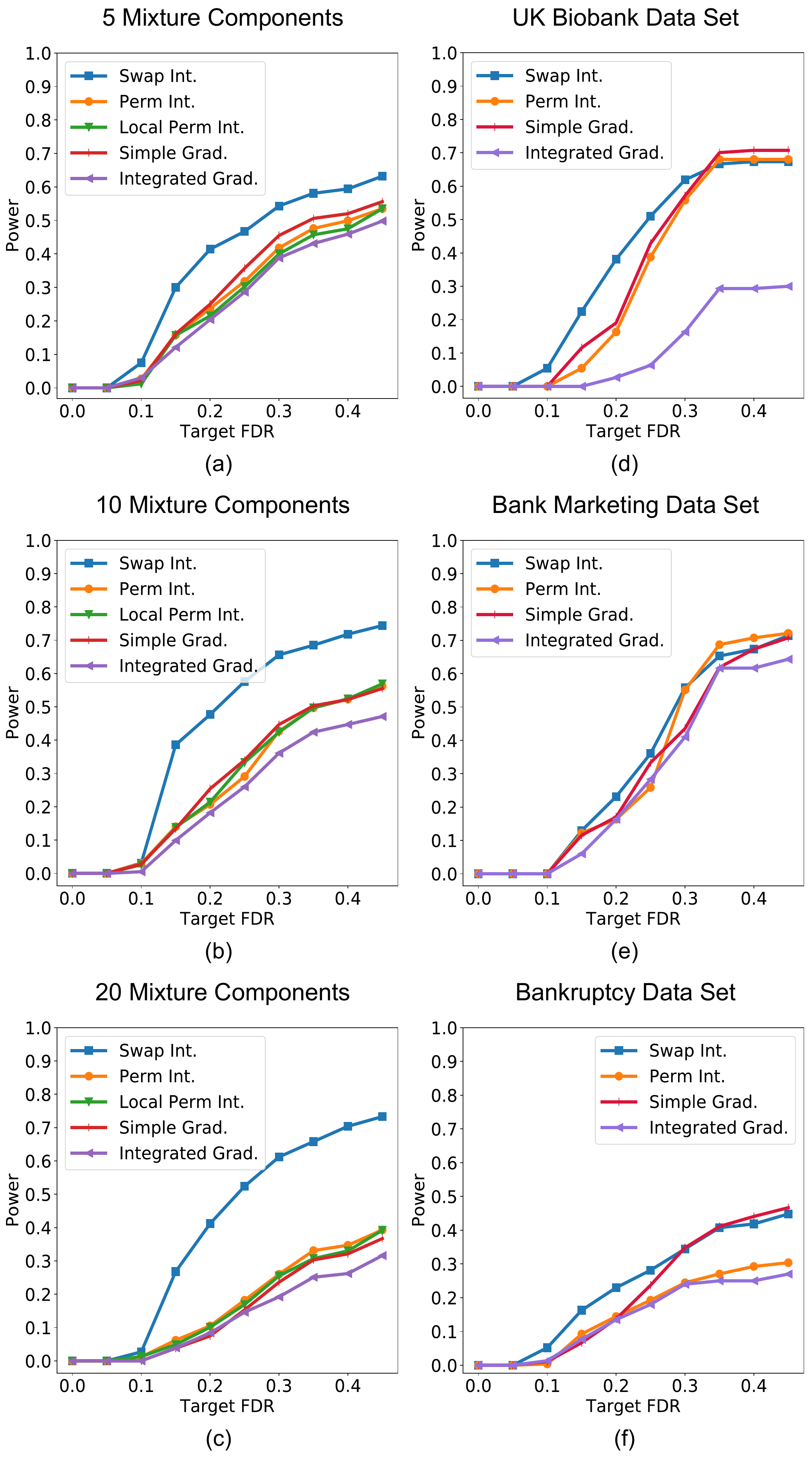}
\vspace{-0.2cm}
\caption{\textbf{Power Comparison Between Importance Score Methods} (a-c) As the number of mixture components increases, the power of Swap Integral method remains high while that of other methods decreases. (d-f)  Power for Swap Integral method is as good as and sometimes better than other methods on real world data with synthetic labels.}
\label{fig:methods}
\end{figure}

We run our knockoff procedure, modeling $P^X$ as a mixture of three multivariate Gaussian variables fitted via EM (Figure~\ref{fig:break}(a)), and as a single Gaussian (Figure~\ref{fig:break}(b)). Importance scores $Z_j$ of the non-null original features are high in both cases, and those of the correlated null original features are also quite high. Importance scores of valid knockoffs cancel out the correlated null original features' importance scores, and therefore the adaptively chosen threshold avoids selecting nulls (Figure~\ref{fig:break}(a)). However, when modeling by one mixture, the knockoffs corresponding to the null features (10-20) have importance scores close to 0 (Figure~\ref{fig:break}(b)). As a result these correlated null features are selected by the procedure and the FDR control is broken. After 40 repetitions of the procedure (randomly generating each run the parameters of $P^X$ and sampling $X$), for a target FDR of $0.2$, we get FDR control when we use our mixture model knockoff (empirical FDR: $0.1823$) and fail to control FDR whenever we fit the data with one single multivariate Gaussian (empirical FDR: $0.5565$).

\paragraph{Feature Selection in Synthetic Data} We generate $P^X$ as mixtures of 5, 10 and 20 multivariate Gaussians, and $Y$ as polynomial functions of a subset of the dimensions of $X$---i.e. these are the non-null features. From each  $P^X$, we draw samples of $X$ and $Y$, fit a mixture of Gaussians and run our knockoff procedure. Figure~\ref{fig:methods} (a-c) shows the power of five importance score methods, and Swap Integral consistently achieves the best power. In all of our experiments, the empirical FDR is controlled at the target level (Appendix Figure~\ref{fig:empiricalFDR}).

\paragraph{Feature Selection in Real Data}

Our real-world simulations are based on three datasets from the UK biobank data and the Bank Marketing and Polish Bankruptcy from the UCI repository \citep{UCI}. We provide more information in Appendix~\ref{appendix:discovery}. As our first experiment, in order to have control over the non-null features, we created synthetic labels for each dataset by randomly selecting a subset of the features to be alternatives (the rest are nulls) and use noisy polynomials on these features to generate labels. Details of the data generation are described in Appendix~\ref{appendix:synthetic}. We fit each dataset with a Gaussian mixture model using the AIC method \citep{aic}. AIC automatically selected 9, 15, and 25 number of clusters for the UK Biobank, Bank Marketing and Bankruptcy datasets, respectively. Then we apply our Gaussian Mixture knockoff sampler to generate knockoffs for each dataset. Average results for repeating this procedure 100 times on each data set are reported in Figure~\ref{fig:methods}. Across the range of target FDRs, the Swap Integral has at least as much power (and sometimes much more) than the other feature statistic methods. It is important to notice that, as we work with a real dataset $X$ that we cannot resample, our experiments cannot provide us with empirical FDR values in order to confirm FDR control. We expect that our model for $P^X$ is accurate enough so that the robustness of the procedure accounts for model inaccuracies.

As a discovery experiment, we run the knockoff procedure for feature selection on each of the three data sets with the real labels --cancer in UK Biobank, telemarketing success in Marketing, and bankruptcy in Bankruptcy. In all cases we train a neural network binary classifier with 4 hidden layers of size 100. All selected features are described in Appendix~\ref{appendix:discovery}. With the real data, we do not have ground truth on which are the true alternate features. However, knockoff selected features make sense and some of the features agree with previously reported literature. With target FDR of 0.3 (it can be shown that for small target FDR values the knockoff procedure requires a minimum number of rejections, which is why we selected a fairly large target FDR), in UK Biobank cancer prediction, the knockoff procedure selected 14 out of the total 284 features. These 14 include \emph{Number of smoked cigarettes daily} and \emph{Wine intake}. Eight out of ten features in Bank marketing data set were selected. For Bankruptcy prediction, the knockoff procedure selected features including \emph{gross profit} and \emph{short-term liabilities}.

\vspace{-0.1cm}
\paragraph{Discussion} The knockoff procedure is a powerful framework for selecting important features while statistically controlling false discoveries. The two main challenges of applying knockoffs in practice are the difficulty in generating valid knockoffs and the choice of importance scores. This paper takes a step toward addressing both challenges. We develop a framework to efficiently sample valid knockoffs from Bayesian Networks. We show that even a simple application of this framework to generate knockoffs from Gaussian mixture models already enables us to apply our procedure with several real datasets, while being optimistic that FDR is controlled. We also systematically evaluate and compare the statistical power of several importance scores. We propose a new score, Swap Integral, which is stable and substantially more powerful. Swap Integral can be applied on top of any classification model, including neural networks and random forests. Note that Swap Integral is a post-processing on top of a single trained classifier and hence is computationally efficient. Our results enable knockoffs to be more practically useful and open many doors for future investigation.

\subsubsection*{Acknowledgments}

J.R.G. and A.G. were supported by a Stanford Graduate Fellowship. J.Z. is supported by a Chan–Zuckerberg Biohub Investigator grant and National Science Foundation (NSF) Grant CRII 1657155.

\bibliographystyle{unsrtnat}

\bibliography{main}

\newpage

\appendix

\section*{APPENDIX}

\section{Proofs}
\label{appendix:proofs}

\subsection{Validity of Importance Scores with Random Component}
\label{appendix:randomfeatures}

Following the notation by \citet{KN2} for Lemma 3.3, denoting $\Wb_{swap(S)}=\Wb([X,\tilde{X}]_{swap(S)},Y)$ the full vector of feature statistics when swapping features in $S$, the \emph{flip-sign property} can be summarized as: $W_{swap(S)} = \epsilon_S \odot W$ where $\odot$ is the element-wise vector multiplication and $\epsilon_S = \mathbb{I}_{j\notin S} - \mathbb{I}_{j \in S}$. As discussed by \citet{KN2}, it should be highlighted that the final selection procedure controls FDR just because of this property. Now, by directly referring to the proof of Lemma 3.3 by \citet{KN2}, we observe that it relies on the flip-sign property just as an equality in distribution. Therefore, with this exact same proof, we get that the result still holds when feature statistics $W$ satisfy the previous equality only in distribution. This allows us to construct valid feature statistics $W$ based on random components that are not limited to the randomness in the data itself. We can therefore also construct randomized $Z$ statistics, and we prove that the constraint mentioned earlier only needs to hold in distribution to end up with $W$ satisfying the flip-sign condition in distribution.

\begin{proposition}
Assume that the following equality holds in distribution for any subset ${S \subset \{1,\dots,d\}}$:
\[
Z([X,\tilde{X}]_{swap(S)},Y) \; \stackrel{d}{=} \;Z([X,\tilde{X}],Y)_{swap(S)}
\]
Then we have the following equality in distribution:
\begin{equation}
W_{swap(S)} \; \stackrel{d}{=} \; \epsilon_S \odot W
\end{equation}
\end{proposition}
\begin{proof}
It suffices to show the result for $S=\{1\}$, as the general case can be decomposed as the concatenation of swaps of just one coordinate.
\begin{align*}
W(&[X,\tilde{X}]_{swap(S)},Y)=
\\ 
&  \begin{bmatrix}
   f_1(Z_1([X,\tilde{X}]_{swap(S)},Y),\tilde{Z}_1([X,\tilde{X}]_{swap(S)},Y)) \\
   f_2(Z_2([X,\tilde{X}]_{swap(S)},Y),\tilde{Z}_2([X,\tilde{X}]_{swap(S)},Y)) \\
 \vdots \\
    f_d(Z_d([X,\tilde{X}]_{swap(S)},Y),\tilde{Z}_d([X,\tilde{X}]_{swap(S)},Y))
\end{bmatrix}
\\ \stackrel{d}{=} & 
\begin{bmatrix}
   f_1(\tilde{Z}_1([X,\tilde{X}],Y),Z_1([X,\tilde{X}],Y)) \\
   f_2(Z_2([X,\tilde{X}],Y),\tilde{Z}_2([X,\tilde{X}],Y)) \\
 \vdots \\
    f_d(Z_d([X,\tilde{X}],Y),\tilde{Z}_d([X,\tilde{X}],Y))
\end{bmatrix}
\end{align*}
\begin{align*}
\\ \stackrel{d}{=} & 
\begin{bmatrix}
   -f_1(Z_1([X,\tilde{X}],Y),\tilde{Z}_1([X,\tilde{X}],Y)) \\
   f_2(Z_2([X,\tilde{X}],Y),\tilde{Z}_2([X,\tilde{X}],Y)) \\
 \vdots \\
    f_d(Z_d([X,\tilde{X}],Y),\tilde{Z}_d([X,\tilde{X}],Y))
\end{bmatrix}
\\ \stackrel{d}{=} &\;  \epsilon_S \odot W([X,\tilde{X}],Y)
\end{align*}

\end{proof}

\subsection{Proof of Proposition~\ref{proposition:GMM}: GMM Knockoff Sampling Procedure}
\label{proof:gaussianmixture}

We prove Proposition~\ref{proposition:GMM}, although this exact same proof applies in the more general setting of the Algorithm~\ref{algorithm:EasyKnockoffAlgo} in next section.

\begin{proof}
We consider the marginal distribution over $(X,\tilde{X})$ by summing the full joint distribution over all possible values of $K$. We then decompose the joint distribution along the sampling steps.
\begin{align*}
& P(X ,\tilde{X}) =  \sum_{k=1}^l P(X,\tilde{X},K=k) 
\\ =&  \sum_{k=1}^l Q^{\tilde{X}|X,K}\!(\tilde{X}|X,K=k)P(X,K=k) 
\\=& \sum_{k=1}^l Q^{\tilde{X}|X,K}\!(\tilde{X}|X,K=k) P^{X|K}\!(X|K=k)P(K=k)
\\ = & \sum_{k=1}^l Q^{X,\tilde{X}|K}(X,\tilde{X}|K=k) P(K=k) 
\end{align*}

This proves exchangeability as the last line satisfies exchangeability in $(X,\tilde{X})$.
\end{proof}

\subsection{Comparison of Algorithm~\ref{algorithm:BNKnockoffAlgo} and SCIP}
\label{proof:differenceSCIP}

The main contribution of our Bayesian network knockoff sampling method is due to the intractability of SCIP for a general feature distribution $P^X$ as mentioned in \ref{section:GeneratingKnockoffs}.

Indeed, SCIP sequentially samples for $1\! \leq \! i \! \leq \! d$ the knockoff $\tilde{X}_i$ of the $i$th feature from the conditional distribution of $X_i$ given $X_{-i},(\tilde{X}_j)_{j < i}$. That means that, at each step of the sampling process, for each sample, one needs to compute the joint distribution of $X,(\tilde{X}_j)_{j<i}$ and the conditional distribution of $X_i$, which is not computationally feasible if we assume a complex model for $P^X$.

Another difference is shown by the following: suppose that we observe variable $X$ for which we want to sample a knockoff, and that its distribution is conditioned on a latent variable $H$. If we assume that we can construct easily the conjugates $Q^{\tilde{X}|X,H}(\tilde{x}|x,h)$ and $Q^{\tilde{H}|H,X}(\tilde{h}|h,x)$, then Algorithm~\ref{algorithm:BNKnockoffAlgo} simplifies into Algorithm~\ref{algorithm:EasyKnockoffAlgo}.

\begin{algorithm}
Sample $H \sim P^{H|X}(\cdot|X)$\;
Sample $\tilde{H} \sim Q^{\tilde{H}|H,X}(\cdot|H,X)$\;
Sample $\tilde{X} \sim Q^{\tilde{X}|X,H}(\cdot|X,\tilde{H})$\;
\caption{Knockoff Sampling Procedure for a simple Latent Variable Model}\label{algorithm:EasyKnockoffAlgo}
\end{algorithm}

We show in this simple setting that $(\tilde{H}, \tilde{X})$ is not a knockoff of $(H,X)$. We write the joint distribution and decompose it along the sampling steps.

\begin{align*}
P(H,&X,\tilde{H},\tilde{X}) 
\\ =& \: Q^{\tilde{X}|X,H}(\tilde{X}|X,\tilde{H})Q^{\tilde{H}|H,X}(\tilde{H}|H,X)
\\ & \hspace{40mm} P^{H|X}(H|X)P^{X}(X)
\\ = & \: Q^{\tilde{X}|X,H}(\tilde{X}|X,\tilde{H})Q^{\tilde{H}|H,X}(H|\tilde{H},X)
\\ & \hspace{40mm} P^{H|X}(\tilde{H}|X)P^{X}(X)
\end{align*}
To prove exchangeability of $(X,\tilde{X})$, we have to marginalize over the hidden states.

\begin{align*}
\\& P(X,\tilde{X})=  \sum_{H,\tilde{H}}P(X,\tilde{X},H,\tilde{H})  
\\ =&  \sum_{H,\tilde{H}}Q^{\tilde{X}|X,H}(\tilde{X}|X,\tilde{H})Q^{\tilde{H}|H,X}(H|\tilde{H},X)
\\ & \hspace{40mm} P^{H|X}(\tilde{H}|X)P^{X}(X)
\\ = & \sum_{\tilde{H}}Q^{\tilde{X}|X,H}(\tilde{X}|X,\tilde{H})P^{H|X}(\tilde{H}|X)P^{X}(X)
\\ = & \sum_{\tilde{H}}Q^{\tilde{X}|X,H}(\tilde{X}|X,\tilde{H})P^{X|H}(X|\tilde{H})P^{H}(\tilde{H})
\end{align*}

Only if we marginalize out the hidden states we get to an expression where exchangeability is satisfied for $(X,\tilde{X})$. Otherwise we don't, and therefore $(\tilde{H}, \tilde{X})$ is not a knockoff of $(H,X)$. In SCIP, all the random variables sampled are part of the final knockoff sample.

Our procedure also differs from SCIP insofar it is ``modular'' in each local conjugate conditional. The choice of each conjugate conditional is not unique, and poor choices yield local knockoffs that are too ``close'' to the initial sample and decrease the power of the procedure. The worst option, which is using the feature as its own knockoff (i.e. $Q^{\tilde{i}|i,MB(i)}(\tilde{x}_i|x_{MB(i)},x_i) = \delta_{x_i = \tilde{x}_i}$) still gives valid knockoffs, though discards any possibility for that given feature to be selected. But this is why this procedure is flexible: in cases where a conditional $P^{i|MB(i)}$ has no closed form expression because of complex dependencies, we can locally choose poor conjugates and continue the procedure so that we still obtain valid knockoff samples, which is not possible when running SCIP directly as one has to sample from a complex conditional distribution that is predetermined.

We can analyze how the previous examples make use of this freedom in the choice of the conditional conjugate. In the Gaussian mixture case we could choose $Q^{\tilde{X}|X,H}(\tilde{X}|X,H)=\delta_{X=\tilde{X}}$, in which case our knockoff for $X$ would be $X$ itself, and the knockoff procedure would be powerless. In the HMM setting, after sampling the hidden nodes from the posterior given the observed ones, we could choose to keep those same nodes as local knockoffs instead of sampling a different local knockoff. In this case, the final knockoff $\tilde{X}$ we obtain is different from $X$. However, one can expect this choice to produce knockoffs with lower power, as the knockoff samples will stay ``closer'' to the true samples. The intuition is that if we leverage the knowledge we have about the generative process, we can sample more powerful knockoffs.

For the Hidden Markov Model, sample:

\begin{itemize}
\setlength\itemsep{-0.2em}
\item $H \sim P^{H|X}(\cdot|X)$, we sample the hidden states. Conditionally on $X$ the distribution of $H$ is that of a Markov chain. \item $\tilde{H} \sim Q^{\tilde{H}|H,X}(\cdot|H,X)$ we sample a new knockoff Markov chain via SCIP. 
\item $\tilde{X} \sim Q^{\tilde{X}|X,H}(\cdot|X,\tilde{H})$. However, $Q^{\tilde{X}|X,H}(\tilde{x}|x,h)$ is constructed based on the distribution $P^{X|H}(x|h)$. But because of the structure of the HMM, the observed states are independent conditionally on the hidden states. If the observed states are univariate then we can simplify the conjugate conditional and sample ${\tilde{X} \sim Q^{\tilde{X}|X,H}(\cdot|X,\tilde{H})}\!=\!P^{X|H}(\cdot|\tilde{H})$ (see comments after Definition~\ref{definition:Conjugate}).
\end{itemize} 

\paragraph{Example: Sampling Knockoffs for LDA} As an additional example, we describe how Algorithm~\ref{algorithm:BNKnockoffAlgo} works to generate knockoffs from Latent Dirichlet Allocation (LDA). 
We use the same notation for LDA as defined by \citet{LDA}: the $n$th word $W_{dn}$ in document $d$ is sampled from a multinomial distribution parametrized by $\beta_{Z_{dn}}$, where $Z_{dn}$ corresponds to the topic assignment for $W_{dn}$. Topic assignment is  sampled from a multinomial distribution parametrized by $\theta_d$, the distribution of topics in document $d$. Finally $\theta_d$ for each document is sampled from a Dirichlet distribution with hyperparameter $\alpha$.
\vspace{-3mm}

\begin{enumerate}[leftmargin=*]
\setlength\itemsep{-0.2em}
\item The first step to build knockoffs is to learn the parameters $\alpha, \beta$ of the model, which can be done by variational EM \citep{LDA}. Then, we need to sample the hidden variables $Z_{dn},\theta_d$ given the observed ones $W_{dn}$ : this is an inference problem for which direct computation is intractable, but we can approximate that posterior distribution via standard variational Bayes methods. 
\item Next is to sample local knockoffs. This is exactly analog to one pass of Gibbs sampling over the whole DAG following a topological ordering, except that instead of sampling with respect to the conditional distribution of the node given its Markov blanket, we sample from the conjugate conditional distribution, conditioning on the appropriate variables as explained in Algorithm~(\ref{algorithm:BNKnockoffAlgo}). We sample each $\theta_d$ based on the conjugate conditional distribution. However, as $\theta_d$ is Dirichlet, any given coordinate is determined by the others, so the only possible choice is to set $\tilde{\theta}_d=\theta_d$. Then, as $Z_{dn}$ is univariate, its conjugate conditional simplifies too so that we just sample from the local conditional probability, and so on for $W_{dn}$.
\end{enumerate} 

\subsection{Proof of Theorem~\ref{theorem:BNKnockoffThm}: DAG Knockoff Sampling Procedure}
\label{proof:DAGknockoff}

\begin{proof}
The joint probability distribution can be decomposed as follows by following the sampling steps:

\begin{align}
P(X,\tilde{X}) =  P(X)\prod_{i=1}^m & Q^{\tilde{i}|i,MB(i)} (\tilde{X}_{i} | \tilde{X}_{\{1:i-1\}\cap MB(i)}, \nonumber
\\ & \hspace{5mm} X_{\{i+1:m\}\cap MB(i)},X_i)\label{equation:eq0}
\end{align}

In order to show that $(X_O,\tilde{X}_O)$ is exchangeable, we want to show that if we marginalize out this joint distribution with respect to the hidden states $(X_H,\tilde{X}_H)$, we get an exchangeable distribution. 

We first show that, iterating recursively over all the nodes, and summing over all values of $X_H$, we obtain the following expression.

\begin{align}
 \sum_{X_H}&P(X,\tilde{X})= \nonumber
 \\ & P(\tilde{X}) \prod_{i\in O}Q^{\tilde{i}|i,MB(i)}(X_{i} |  \tilde{X}_{\{1:i-1\}\cap MB(i)},\tilde{X}_i) \label{equation:eq1}
\end{align} 

For simplicity, here we consider discrete random variables, so that marginalizing the joint distribution over $X_i$ means summing over all possible values of $X_i$. Everything stays valid for continuous random variables, replacing sums by integrals. Starting from equation (\ref{equation:eq0}), which corresponds to step 1, we do sequentially $m$ steps to get to (\ref{equation:eq1}).  Suppose that at step $1\leq k \leq m$ we have the following equality where the left-hand term is the product of the right-hand terms:

\begin{align*}
& \sum_{\substack{X_l \\  l \in H, \: l \leq k-1} } \hspace{-3mm} P(X,\tilde{X})= P(\tilde{X}_{1:k-1},X_{k:m}) 
 \\ \times & \! \prod_{i\geq k}\!Q^{\tilde{i}|i,MB(i)}(\tilde{X}_{i} |  \tilde{X}_{\{1:i-1\}\!\cap\! MB(i)},\! X_{\{i+1:m\}\!\cap \!MB(i)},\!X_i)  
 \\ \times \hspace{-2mm} & \prod_{\substack{i \in O\\  i \leq k-1}}\hspace{-2mm}Q^{\tilde{i}|i,MB(i)}(X_{i} |  \tilde{X}_{\{1:i-1\}\cap MB(i)},\tilde{X}_i)
\end{align*} 

The key element is that, by following the topological order, at step $k$ the variable $X_k$ only appears in the joint probability $P(\tilde{X}_{1:k-1},X_{k:m})$  and in the term

\[
Q^{\tilde{k}|k,MB(k)}(\tilde{X}_{k} |  \tilde{X}_{\{1:k-1\}\cap MB(k)},X_{\{k+1:m\}\cap MB(k)},X_k)
\] 
(Notice that, if $i\leq k-1$ corresponds to an observed node, it has no descendents. Therefore the Markov blanket of such node is a subset of the nodes with smaller index/topological ordering). We isolate these two terms and start by writing down the joint probability as a conditional probability. By definition of the Markov blanket, we can simplify the expression of the conditional probability. Then, we obtain two terms that are conjugate in the exchangeable sense.

\begin{align*} 
&P(\tilde{X}_{1:k-1},X_{k:m})Q^{\tilde{k}|k,MB(k)}(\tilde{X}_{k} |  \tilde{X}_{\{1:k-1\}\cap MB(k)},
\\[-3pt] & \hspace{3cm} X_{\{k+1:m\}\cap MB(k)},X_k)
\\[7pt]& = \! P^{k|MB(k)}(X_k|\tilde{X}_{\{1:k-1\}\cap MB(k)},X_{\{k+1:m\}\cap MB(k)})
\\[-3pt] & \times \! P(\tilde{X}_{\{1:k-1\}},X_{\{k+1:m\}})
\\[-3pt]& \times\! Q^{\tilde{k}|k,MB(k)}\!(\tilde{X}_{k} |  \tilde{X}_{\{1:k-1\}\!\cap\! MB(k)},\!X_{\{k+1:m\}\!\cap\! MB(k)},\!X_k)
\\[7pt]&  = \! P^{k|MB(k)}(\tilde{X}_k|\tilde{X}_{\{1:k-1\}\cap MB(k)},X_{\{k+1:m\}\cap MB(k)})
\\[-3pt] & \times \! P(\tilde{X}_{\{1:k-1\}},X_{\{k+1:m\}})
\\[-3pt]& \times \! Q^{\tilde{k}|k,MB(k)}\!(X_{k} |  \tilde{X}_{\{1:k-1\}\!\cap\! MB(k)},\!X_{\{k+1:m\}\!\cap\! MB(k)},\!\tilde{X}_k)
\\[7pt] & = \! P(\tilde{X}_{1:k},X_{k+1:m})Q^{\tilde{k}|k,MB(k)}(X_{k} |  \tilde{X}_{\{1:k-1\}\cap MB(k)},
\\[-3pt] & \hspace{3cm} X_{\{k+1:m\}\cap MB(k)},\tilde{X}_k)
\end{align*}

We swap in the previous expression the two terms and we get the following product:

\begin{align*}
&\sum_{\substack{X_l \\  l \in H, \: l \leq k-1} }\hspace{-4mm} P(X,\tilde{X})=  P(\tilde{X}_{1:k},X_{k+1:m})
\\ \times & \hspace{-2mm} \prod_{i\geq k+1}\hspace{-2mm}Q^{\tilde{i}|i,MB(i)}(\tilde{X}_{i} |  \tilde{X}_{\{1:i-1\}\!\cap\! MB(i)},\! X_{\{i+1:m\}\!\cap \!MB(i)},\!X_i)  
\\   \times \hspace{-2mm} &\prod_{\substack{i \in O\\  i \leq k-1}}\hspace{-2mm}Q^{\tilde{i}|i,MB(i)}(X_{i} |  \tilde{X}_{\{1:i-1\}\cap MB(i)},\tilde{X}_i) 
\\  \times & Q^{\tilde{k}|k,MB(k)}(X_{k} |  \tilde{X}_{\{1:k-1\}\!\cap\! MB(k)},\! X_{\{k+1:m\}\!\cap\! MB(k)},\!\tilde{X}_k)
\end{align*} 

If $k\in H$, then we sum both sides of the equality over $X_k$. But now, $X_k$ only appears in the last term, and summing over it gives $1$. If we reach a node with no descendents, i.e. $k \in O$, then we do not marginalize out. However we have the following simplification:

\begin{align*}
Q^{\tilde{k}|k,M\!B(k)}&(X_{k} |  \tilde{X}_{\{1:k-1\}\cap MB(k)},\! X_{\{k+1:m\}\cap MB(k)},\! \tilde{X}_k)\! 
\\ & =   Q^{\tilde{k}|k,M\!B(k)}(X_{k} |  \tilde{X}_{\{1:k-1\}\cap MB(k)},\tilde{X}_k)
\end{align*}

In both cases, we get to the next step in our recursion. After completing last step, we get to equation (\ref{equation:eq1}). Notice that this expression is exchangeable in $X_O, \tilde{X}_O$, for every assignment of $\tilde{X}_H$. Indeed, for $l \in O$, we have that $\{1:l-1\}\cap MB(l) = MB(l)$, so

\begin{align*}
\sum_{X_H}P(X,\tilde{X})=&P(\tilde{X}_{-l})P^{l|MB(l)}(\tilde{X}_l|\tilde{X}_{MB(l)})
\\  \times &  \, Q^{\tilde{l}|l,MB(l)}( X_l|  \tilde{X}_{\{1:l-1\}\cap MB(l)},\tilde{X}_{l})  
\\ \times & \prod_{\substack{i\in O \\  i\neq l} }    Q^{\tilde{i}|i,MB(i)}(X_i |  \tilde{X}_{\{1:i-1\}\cap MB(i)},\tilde{X}_{i})
\end{align*}

And $(X_l,\tilde{X}_l)$ do not appear in the last product term as two observed nodes cannot be in the Markov blanket of each other. $(X_l,\tilde{X}_l)$ only appear in the conjugate probabilities, therefore the exchangeability in $(X_l,\tilde{X}_l)$ holds. Again, as two observed nodes cannot appear in the Markov blanket of the other, this step can be repeated for different indices $l \in O$, hence the exchangeability of the expression. This symmetry is at fixed values of $\tilde{X}_H$. Therefore, it still holds when we sum over $\tilde{X}_H$. Hence 

\[
P(X_O,\tilde{X}_O) = \sum_{\tilde{X}_H,X_H}P(X,\tilde{X})
\]

satisfies exchangeability.
\end{proof}

\section{Robustness of the Procedure to Model Misspecification}
\label{proof:robustness}

\begin{figure}[ht]
\centering
\includegraphics[width=1\linewidth]{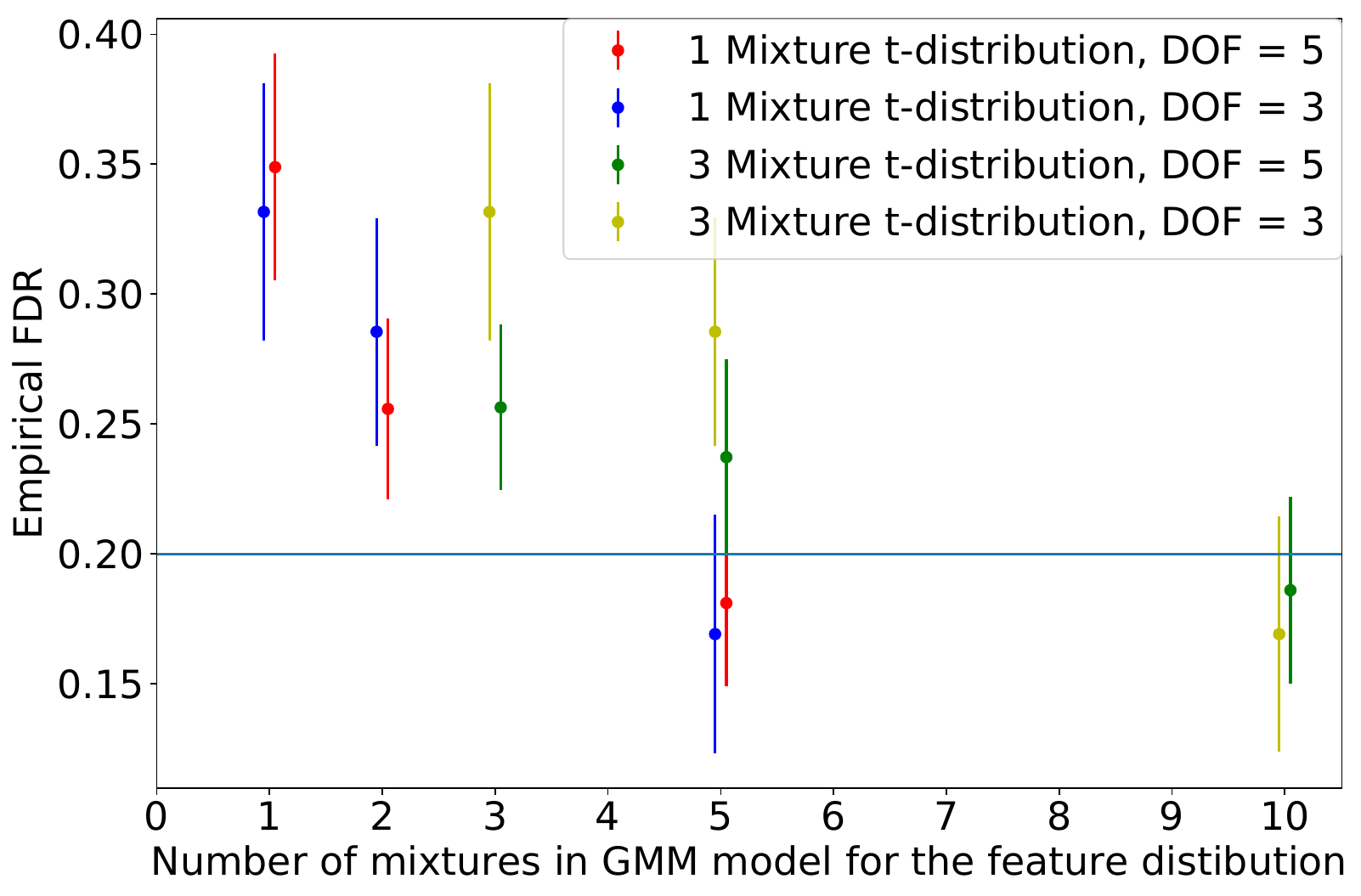}
\caption{\textbf{Fitting a t-distribution with a Mixture of Gaussians} We evaluate the empirical FDR when running the knockoff procedure with a misspecified model. We generate knockoffs by fitting a Gaussian mixture in settings where the features are from a mixture of t-distribution, for different degrees of freedom (DOF).}
\label{fig:t-distribution}
\end{figure}

\begin{figure*}[ht]
\centering
\includegraphics[width=0.8\linewidth]{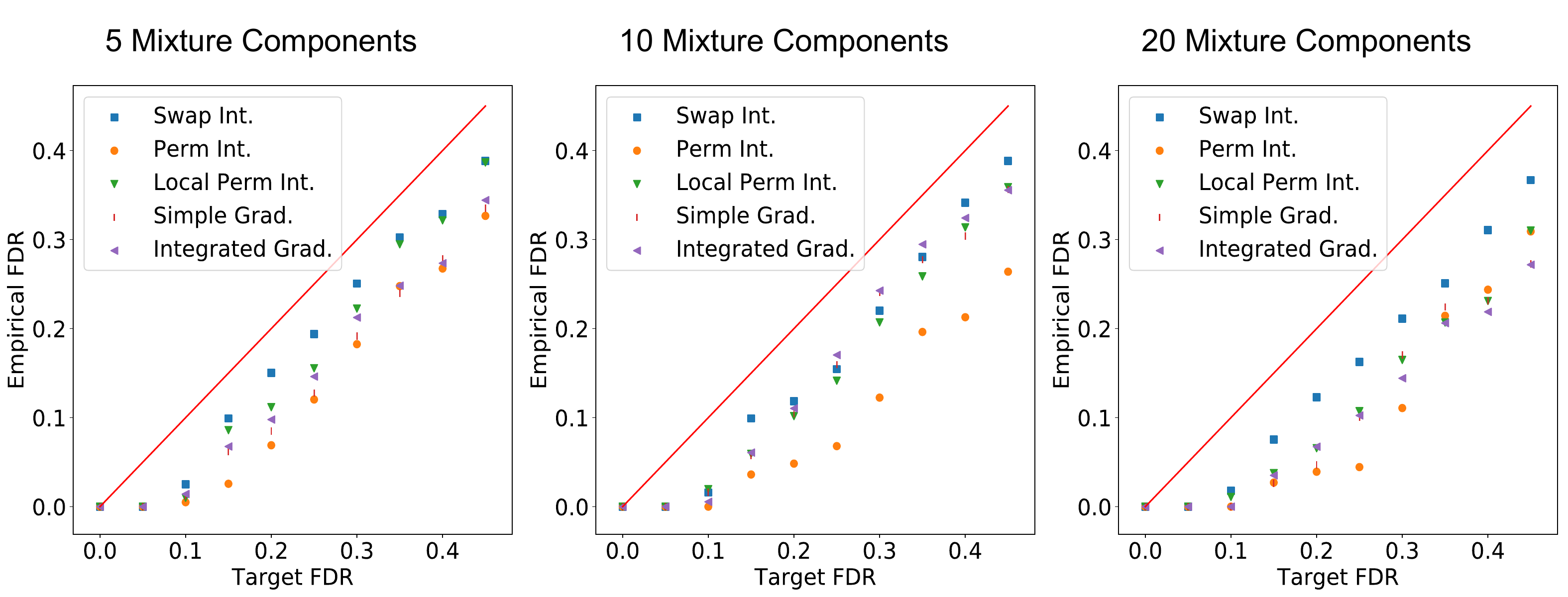}
\caption{\textbf{Empirical FDR for Experiments With Synthetic Features}}
\label{fig:empiricalFDR}
\end{figure*}

As explained when first introducing the \emph{Model-X} knockoff procedure in \citet{KN2}, instead of considering a model for the conditional distribution of $Y|X$, all the assumptions are related to modeling $X$. The \emph{burden of knowledge} shifts from $Y|X$ to $X$. The same way valid p-values rely on assumptions on $Y|X$, (parametric model, noise distribution, asymptotic regime...), valid knockoffs rely on assumptions on the distribution of $X$: mainly, that we can approximate it very well.  

When we generate knockoffs based on a Gaussian mixture model, and more generally a Bayesian network, we assume that these probabilistic models are good approximations for $P_X$, and that they can be properly fitted. This is a very strong assumption, as not only the model we use to represent $X$ may be incorrect, but the estimated parameters of the model depend on the fitting procedure, which sometimes provides only an approximation to the actual distribution encoded by the Bayesian network (as when using Variational Inference). Optimization methods commonly used such as Expectation-Maximization (EM) can also get stuck in local minima. However, the knockoff procedure is remarkably robust when dealing with these issues. Existing theoretical robustness bounds \citep{KN4} are based on controlling the KL-divergence of the model with respect to the true $P_X$. This can help explain why EM in our method, and hopefully other fitting procedures, yield knockoffs that are somehow valid: these methods minimize the KL-divergence of the model with respect to the distribution of $X$.

\section{Synthetic Data Generation and FDR Control}
\label{appendix:synthetic}

As an example, we provide simulations showing the empirical FDR for a mixture of t-distributions in Figure~\ref{fig:t-distribution} (at a fixed target FDR $0.2$), as a function of the number of Gaussian mixtures we use to model the distribution. The conclusion is that, with enough mixtures, the knockoff procedure is able to control FDR, even though we cannot expect to correctly represent any t-distribution through a mixture of Gaussians.

To generate a synthetic data set with $n$ samples, $d$ features, $l$ mixtures and $C$ different classes, we implemented the following steps:

\begin{itemize}[noitemsep, topsep=0pt]
\item We generate random values for the means and covariance matrices (using scikit-learn positive-semidefinite matrix generation function) for each of the $l$ mixtures and the mixture proportions.
\item For each $1\leq i \leq n$ we sample $K_i$ the mixture assignment for sample $i$.
\item We sample $X = (X_{i1},\dots ,X_{id})$ from the Gaussian distribution corresponding to the $K_i$th mixture. 
\item We define $f_c : \mathbb{R}^d \rightarrow \mathbb{R}$  for $c\in\{1,\dots,C\}$ to be 3rd order polynomial functions over the attributes. The coefficients of the polynomial functions are randomly sampled from $\mathcal{N}(0, 1)$.
\item Each sample $X_i$ is labeled by $Y_i = \arg\max_{c\in\{1,\dots,C\}} (f_c(X_i) + \epsilon_{ic})$ where \mbox{$\epsilon_{ic} \sim \mathcal{N}(0,0.1)$} i.i.d.
\end{itemize}
To generate synthetic labels for a real world data set, we go through the same procedure, but without generating the input features. It is crucial to notice that, in these experiments with real data, it is not possible to verify that our method controls FDR, given that we can not obtain new batches of data coming from the same distribution on which to repeat the procedure to get an empirical FDR. These experiments are done for the purpose of power comparison, which remains pertinent even if we only regenerate the synthetic label.

For the experiments where we repeatedly generated the synthetic data $X$, we can verify that our procedure controls FDR by computing an empirical FDR over several runs of the procedure. Figure~\ref{fig:empiricalFDR} plots the empirical FDR vs. the target FDR, whenever $X$ is sampled from different numbers of mixtures, and for all the methods we use to compute feature statistics. 

\section{Selected Features for Real Datasets}
\label{appendix:discovery}

We work on three real world data sets: (1) 17596 randomly sampled participants from the UK Biobank data set \citep{biobank}.
Each individual has 284 phenotype features. (2) Bank Marketing~\citep{bankMarketing} Data Set of  UCI~\citep{UCI}, containing 45211 samples with 10 real-valued features for a binary classification task of bank telemarketing success prediction. (3) Polish bankruptcy dataset~\citep{bankrupcy} of the UCI repository containing 10503 samples with no missing attribute, each with 64 real-valued attributes for a binary task of company bankruptcy prediction.

\paragraph{Disease Prediction} With target FDR=0.3, the following features were selected for the task of Malignant neoplasm of breast with ICD10 code C50:
\begin{itemize}[noitemsep, topsep=0pt]
    \item Duration of walks
    \item Ankle spacing width
    \item Average weekly champagne plus white wine intake
    \item Coffee intake
    \item Number of cigarettes previously smoked daily
    \item Interval between previous point and current one in numeric path (trail \# 1) (related to intelligence question results)
    \item Father's age at death
    \item Longest period of depression
    \item Particulate matter air pollution 
    \item Inverse distance to the nearest road
    \item Number of days/week walked +10 minutes
    \item Mean reticulocyte volume
    \item Length of menstrual cycle
    \item Average weekly spirits intake
\end{itemize}
\paragraph{Bank Marketing Success Prediction} With target FDR=0.3:
\begin{itemize}[noitemsep, topsep=0pt]
\item age
\item duration
\item campaign
\item pdays
\item previous
\item emp.var.rate
\item cons.price.idx
\item cons.conf.idx
\item euribor3m'
\item  nr.employed
\end{itemize}
\paragraph{Bankruptcy Prediction} With target FDR=0.3: 
\begin{itemize}[noitemsep, topsep=0pt]
    \item Gross profit (in 3 years) / total assets
    \item Profit on sales / total assets
    \item Retained earnings / total assets
    \item Gross profit / short-term liabilities
\end{itemize}

\section{Intuition Behind Drawback of Permutation Importance Scores}
\label{appendix:schematic}
\begin{figure}[ht]
\centering
\includegraphics[width=0.5\linewidth]{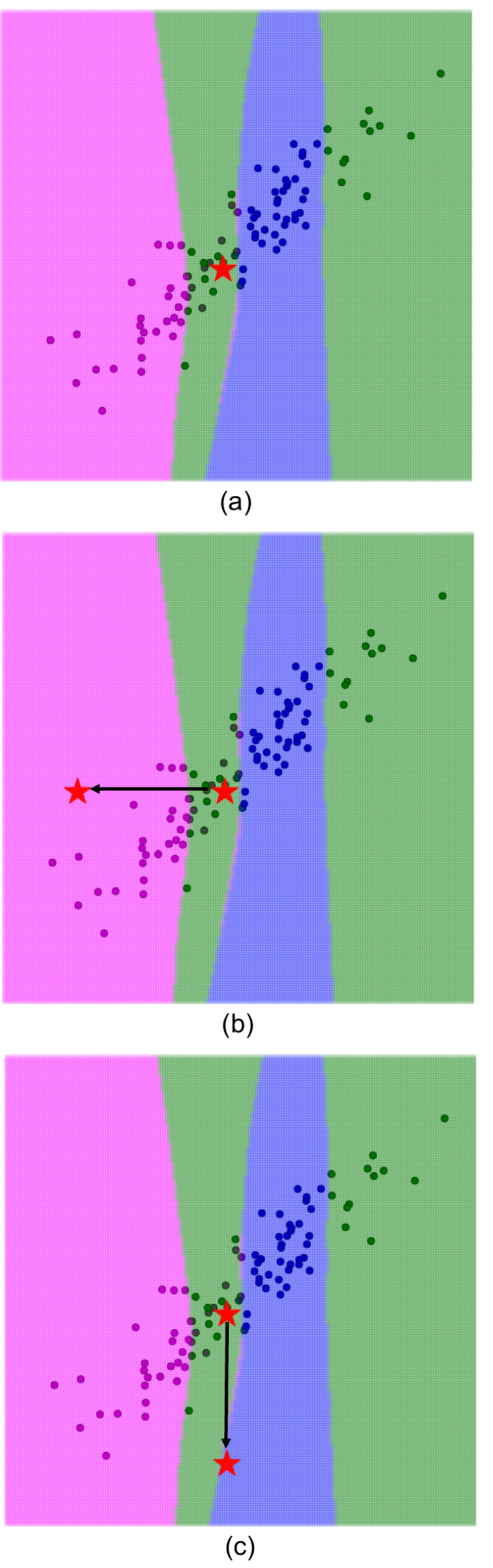}
\caption{\textbf{Drawback of Permutation Method for Importance Score} }
\label{fig:permutation}
\end{figure}
We explain the phenomenon through Figure~\ref{fig:permutation}.

(a) A neural network is trained with a dataset with one feature concatenated with its generated knockoff feature. The horizontal axis corresponds to the original and the vertical axis is the knockoff feature. The decision boundaries of the trained network are displayed. 

(b) Applying shuffling to one of the samples in its original feature will result in an incorrect prediction and therefore a high importance score for the original feature. 

(c) Although the knockoff feature has no effect on prediction, as applying shuffling results in an \emph{off-distribution} fake data point, the predicted label of the fake data point will be incorrect again as it lies in part of the input space that the network has not been trained on. The importance score for both the original and the knockoff feature will both be high, which will result in a small feature statistic and therefore prevent that non-null feature from being selected.

\end{document}